\documentclass{article} 
\usepackage{nips12submit_e,times}
\usepackage{verbatim}
\usepackage{bbm}
\usepackage{amssymb}
\usepackage{amsthm}
\usepackage{amsmath}
\usepackage{graphicx}
\usepackage{fancyhdr}
\usepackage{wrapfigure}
\usepackage{url}
\newenvironment{compactitemize}{
  \begin{list}{\labelitemi}{\leftmargin=1.5em}
    \setlength{\itemsep}{1.5pt}
    \setlength{\parskip}{0pt}
    \setlength{\parsep}{0pt}}{\end{list}}
\newenvironment{compactenumerate}{
  \begin{list}{\labelitemi}{\leftmargin=1.5em}
    \setlength{\itemsep}{1.5pt}
    \setlength{\parskip}{0pt}
    \setlength{\parsep}{0pt}}{\end{list}}
\newcommand{\defeq}{\mbox{$\;\stackrel{\mbox{\tiny\rm def}}{=}\;$}}
\newcommand{\leqa}{\mbox{$\;\stackrel{\mbox{\rm (a)}}{\leq}\;$}}
\newcommand{\leqb}{\mbox{$\;\stackrel{\mbox{\rm (b)}}{\leq}\;$}}
\newcommand{\leqc}{\mbox{$\;\stackrel{\mbox{\rm (c)}}{\leq}\;$}}
\newcommand{\leqd}{\mbox{$\;\stackrel{\mbox{\rm (d)}}{\leq}\;$}}

\newcommand{\sgn}{\text{sgn}}
\newcommand{\Vol}{\text{Vol}}
\newcommand{\abs}[1]{\left| #1\right|}
\def\ignore#1{}

\newtheorem{theorem}{Theorem}

\newtheorem{lemma}[theorem]{Lemma}

\newtheorem{proposition}[theorem]{Proposition}
\newtheorem{definition}{Definition}

\def\cD{\mathcal{D}}

\def\bbR{{\mathbb R}}

\def\bbP{\mathbb{P}}
\def\bbE{\mathbb{E}}
\def\bbm1{\mathbbm{1}}
\def \bbB{\mathbb{B}}

\def\abs{}

\newcommand{\ra} {{\rangle}}
\newcommand{\la} {{\langle}}

\def\cD{\mathcal{D}}

\def\cX{\mathcal{X}}

\def\cR{\mathcal{R}}

\def\01{\{0,1\}}
\def\r01{[0,1]}

\def\Rreg{\cR^{\text{reg}}}
\def\Rx0{\cR}
\def\R01{\cR_{0-1}}

\def\wunderstarreg{w_{*}^{\text{reg}}}
\def\hatR{\hat{\cR}}
\def\hatRreg{\hat{\cR}^{\text{reg}}}
\def\hatwunderstarreg{\hat{w}_*^{reg}}
\def\hatwunderstarminusireg{\hat{w}^{-i,\text{reg}}_{*}}

\def\wdotx{\langle w,x\rangle}
\def\wdotxi{\langle w,x_i\rangle}
\def\wstarLx0lambdasigma{w^{*}_{L,x_0,\lambda,\sigma}}
\title{Local Support Vector Machines: \\Formulation and Analysis}
\author{ Ravi Ganti, Alexander Gray\\
School of Computational Science \& Engineering,
Georgia Tech\\
gmravi2003@gatech.edu, agray@cc.gatech.edu}
\nipsfinalcopy
\begin{document}
\maketitle
  \begin{abstract}
We provide a formulation for Local Support Vector Machines (LSVMs) that generalizes previous formulations, and brings out the explicit connections to local polynomial learning used in nonparametric estimation literature. We investigate the simplest type of LSVMs called Local Linear Support Vector Machines (LLSVMs). For the first time we establish conditions under which LLSVMs make Bayes consistent predictions at each test point $x_0$. We also establish rates at which the local risk of LLSVMs converges to the minimum value of expected local risk at each point $x_0$. Using stability arguments we establish generalization error bounds for LLSVMs.
  \end{abstract}
  \section{Introduction}
  We consider the problem of binary classification, where we are given a sample $S$ of $n$ i.i.d points, $S=\{(x_1,y_1),\ldots,(x_n,y_n)\} \in (\cX\times \{-1,1\})^n$, $\cX\subset\bbR^d$, and we are required to learn a classifier $g_n:\cX\rightarrow \{-1,+1\}$ using $S$. Binary classification is a well studied problem in machine learning~\cite{hastie2003esl,boucheron2005theory}. One of the simplest classification algorithm is the k-nearest neighbour (kNN) algorithm. The kNN algorithm takes a majority vote over the $k$ nearest neighbours of a test point $x_0$ in order to determine the label of $x_0$. The kNN algorithm, among others, belongs to the class of local learning algorithms that take into account only the local information around the test point $x_0$ in deciding the label of $x_0$. Cortes and Vapnik~\cite{vapnik1998statistical} introduced the celebrated support vector machine (SVM), which learns a decision boundary by maximizing the margin. Zhang et al.~\cite{zhang2006svm}, and Blanzieri~et al.~\cite{blanzieri2006adaptive} independently proposed a classification algorithm called \textit{Local Support Vector  Machines} (LSVMs) and applied it to remote sensing and visual recognition tasks respectively. LSVMs exploit locality, like kNN,  along with the idea of large margin classification, to learn a global non-linear classifier, by learning an SVM locally at each test point. By using a large margin approach LSVMs inherit large margin classifier's robustness to data perturbation. By utilizing only local information, LSVMs avoid relying on the global geometry of the distribution. This is a good strategy when our data lies on a manifold, where the geodesic distance between close points is approximately Euclidean, but the same is not true for points far away. A prime example is the task of image recognition. 

 Segata et al.~\cite{segata2010fast} compared their implementation of approximate LSVM with a standard RBF SVM in LIBSVM for the 2-spirals dataset. The 2-spirals dataset is a 2-dimensional dataset where the data lives on a manifold. The accuracy of LSVM was 88.47\% whereas that of SVM was only 85.29\%, and that of kNN was 88.43\%. 
Another scenario where local learning is more beneficial than global learning is when data is multimodal and/or heterogeneous. For example, suppose given census data we are required to classify people as belonging to the high income group (HIG) or the middle income group (MIG). A global HIG vs MIG classifier might be hard to build since the notion of HIG/MIG changes with states/counties. However, it is better to utilize local information, such as information from a particular county, to build multiple local  classifiers. Our global classifier is then a collection of many such local classifiers.The problem of webspam detection is another example where data is heterogeneous. A page might be webspam for one category but may not be for another. In such cases in order to categorize a webpage as spam or not, it is easier to build multiple local classifiers rather than a single global classifier. A numerical illustration was provided by Cheng et al.~\cite{cheng2009efficient} On a one-vs-all classification problem on the Covtype dataset, their implementation of LSVM registered an accuracy of about 90\%, whereas the accuracy of RBF SVM was only 86.21\% and that of kNN was 67.40\% . Since a one-vs-all classification problem makes the dataset multimodal and heterogeneous (due to grouping of multiple classes as one single class), the superior performance of LSVMs over SVMs illustrates the power of local learning in such settings.  A practical advantage of LSVMs is that they can exploit fast algorithms for range search~\cite{gray_nbody}, and various other approximation techniques~\cite{segata2010fast}, along with parallel computing architectures in order to learn on large datasets. 


SVMs are well understood both theoretically and practically~\cite{boucheron2005theory, vapnik1998statistical}. However, no theoretical understanding yet exists for LSVMs. The empirical success of LSVMs begs a theoretical understanding of such techniques. Our work is the first attempt to provide a theoretical understanding of LSVMs. Our contributions are as follows:
\vspace{-5pt}
\begin{compactenumerate} 
  \item We provide a  formulation of LSVMs, which generalizes previous formulations due to Blanzieri et al.~\cite{blanzieri2006adaptive}, and Zhang et al.~\cite{zhang2006svm}, and provide the first statistical analysis of locally linear support vector machines (LLSVMs), the simplest kind of LSVMs. Our formulation (Section~\eqref{sec:formulation}) is similar to Cheng et al.~\cite{cheng2009efficient} but  is more directly motivated from local polynomial regression, and hence the role of a smoothing kernel in our formulation, is much more cleaner than their use of an unspecified weight function. Our formulation makes explicit the direct connections to local polynomial fitting via the use of polynomial Mercer kernels. This allows us to view LSVMs as approximating the decision boundary locally using smooth functions, a novel interpretation.

\item In Theorem~\eqref{thm:main} (Section~\eqref{sec:mainthm}) we provide sufficient conditions, which guarantee that, for any given point $x_0$, the prediction of LLSVMs at $x_0$ matches that of the Bayes classifier, establishing pointwise consistency. These are conditions on the distribution and the model parameters $\lambda,\sigma$. 

\item The LLSVM problem at any point $x_0$ minimizes the sample version of the stochastic objective: $\min_{w} \frac{\lambda}{2}||w||^2+\bbE  L(y\langle w,x\rangle)K(x,x_0,\sigma)$. In Theorem~\eqref{thm:stochastic} (Section~\eqref{sec:theory_contrib}) we provide a high probability bound of $\tilde{O}(\frac{1}{\sqrt{n}\lambda\sigma^{2d}})$ for the difference between the smallest value of this stochastic objective and, its value at the solution obtained by solving the LLSVM optimization problem. This result tells us how quickly the stochastic objective, at the solution of the LLSVM problem, converges to its true minimum value. 

\item Define the L risk of any function $f$ as $\bbE L(yf(x))$. In Theorem~\eqref{thm:risk_bound_for_llsvm}  (Section~\eqref{sec:theory_contrib}) we establish an upper bound on the gap between the $L$ risk of the global function learnt by LLSVMs, with bandwidth $\sigma$, and regularization $\lambda$, and the empirical $L$ risk of LLSVMs, via uniform stability bounds. This gap decays as $O(\frac{1}{\sqrt{n\lambda\sigma^d}})$. Notice that while theorems~\eqref{thm:main},\eqref{thm:stochastic} are pointwise results, theorem~\eqref{thm:risk_bound_for_llsvm} involves the global classifier learnt by solving the LLSVM problem at each training point $x_i$ with parameters $\lambda,\sigma$. Hence this is a ``global'' result. Theorems~\eqref{thm:main},~\eqref{thm:risk_bound_for_llsvm} suggest that LLSVMs should work well in low dimensions, or if the data lies in a low-dimensional manifold. This justifies the empirical findings of Segata et al., and Cheng et al.
\end{compactenumerate}
\section{Formulation of LLSVMs and LSVMs.\label{sec:formulation}}  
Our formulation for LLSVMs is directly motivated from a certain technique in nonparametric regression called local linear regression (LLR)~\cite{tsybakov2009introduction}. In LLR one fits a non-linear regression function by fitting a linear function locally at each point. The idea of local linear fit is inspired by the fact that any differentiable function can be well approximated locally via linear functions. Hence LLR locally approximates the underlying regression function with a linear function. LLSVMs adopt a similar approach by making local linear fits to the underlying non-linear decision boundary.  In order to classify an unseen point $x_0$, LLSVMs solve the problem
\vspace{-1pt}
  \begin{equation}
  \label{eqn:llsvm_opt}		
    \textstyle{\hatwunderstarreg=\arg \underset{w}{\min}\quad\frac{\lambda}{2}||w||^2+\frac{1}{n}\sum_{i=1}^n L(y_i\langle w,x_i\rangle)K(x_i,x_0,\sigma)}, 
\vspace{-1pt}
\end{equation}
where $K(x_i,x_0,\sigma)$ is a smoothing kernel with bandwidth $\sigma$ and $L(\cdot)$ is a Lipschitz, convex upper bound to the 0-1 loss. In this paper we will be concerned with the hinge loss  $L(t)\defeq\max\{1-t,0\}$, which is used in SVMs. Some popular examples of smoothing kernels\footnote{Note that smoothing kernels are not the same as the Mercer kernels used in SVM. A popular example of a smoothing kernel that is not a Mercer kernel is the Epanechnikov kernel.} are the Epanechnikov kernel, the rectangular kernel ~\cite{tsybakov2009introduction}.
Replacing the term $L(y_i\la w,x_i\ra)$ in equation~\eqref{eqn:llsvm_opt} with the term $L(y_i\la w,\phi(x_i)\ra)$, where $\phi(x_i)$ is a kernel map induced by a Mercer kernel, we get LSVMs. 
  \begin{equation}
    \label{eqn:lsvm_opt}		
    w^{*}=\arg \underset{w}{\min}\quad\frac{\lambda}{2}||w||^2+\frac{1}{n}\sum_{i=1}^n L(y_i\langle w,\phi(x_i)\rangle)K(x_i,x_0,\sigma). 
  \end{equation}
Our formulation of LSVMs as shown in equation~\eqref{eqn:llsvm_opt} strictly generalizes the formulation of both  Blanzieri~et al. and Zhang et al.. Strictly speaking, their algorithm used a rectangular smoothing kernel with the bandwidth equal to the distance of the $k^{\text{th}}$ nearest neighbour of the test point $x_0$ in the training set. In comparison  our formulation uses a smoothing kernel that allows the formulation to down-weight points in a smooth fashion. The vector $\hatwunderstarreg$ that is learnt by solving the optimization problem~\eqref{eqn:llsvm_opt} is used for classification at $x_0$ only. Hence, unlike linear SVMs, LLSVMs are still non-linear as the linear fits are only local, and the smoothing kernel precisely determines the locality at each $x_0$. To see a simple example of the influence of a smoothing kernel, consider LLSVMs with the hinge loss. Standard primal-dual calculations yield  
\begin{equation}  \label{eqn:primal_dual}
    \hatwunderstarreg=\frac{1}{\lambda}\sum_{i=1}^n \alpha_iy_iK(x_i,x_0,\sigma)x_i,0\leq \alpha_i\leq 1/n
\end{equation}
where $\alpha_i$ are the dual variables.
If one uses a finite tailed smoothing kernel such as an Epanechnikov kernel or a rectangular kernel, then  the points $x_i$ which are outside the bandwidth of the kernel, i.e. $K(x_i,x_0,\sigma)$=0, have no effect on $\hatwunderstarreg$. Hence, the resulting LLSVM does not care about these points and tries to maximize the margin in the input space using points that are close to $x_0$. 

Mercer kernels, that arise out of the kernel map $\phi(\cdot)$, on the other hand have nothing to do with locality. Instead they allow us to fit non-linear functions. If one uses a polynomial Mercer kernel of degree $d$, in conjunction with a smoothing kernel,  then it is equivalent to making local degree $d$ approximations to the boundary function. While such approximations are potentially more powerful than local linear approximations, one would require stronger conditions such as existence of higher order derivatives of the decision boundary, to justify local polynomial approximations. To avoid making such strong assumptions we shall focus on LLSVMs in this paper. The formulation of Cheng et al.~\cite{cheng2009efficient} is similar to the optimization problem ~\eqref{eqn:lsvm_opt}, but uses an unspecified weight function, $\sigma(x_i,x_0)$, in the place of $K(x_i,x_0,\sigma)$. While the importance of the smoothing kernel, and its interaction with Mercer kernels has been distilled in our formulation, the importance and impact of the weight function in the formulation of Cheng et al. was not done clearly.

\textbf{Related Work.} Kernel based rules (KBRs) have been proposed as a nonparametric classification method (see chapter 10 in~\cite{devroye1996ptp}) and are essentially a simplified version of LLSVMs. KBRs predict the label of a point $x_0$ as  $\sgn(\sum_{i=1}^n y_iK(x_i,x_0,\sigma))$. This can be seen as using equation~\eqref{eqn:primal_dual} but with all $\alpha_i$'s set to a constant. However, for LLSVMs these $\alpha$ values themselves depend on the training data, and hence results from the KBRs literature do not transfer to our case. Learning multiple local classifiers has also been done by first clustering the data and then learning a classifier in each of these clusters~\cite{kim2005locally,yan2004discriminant}, or by using a baseline classifier~\cite{dai2006locally,dekelthere} to find regions where the classifier commits errors and then learning a dedicated classifiers for each of these erroneous regions. All these algorithms are different from LLSVMs as they learn a finite mixture of local classifiers from the training data and classify the test point as per the appropriate mixture component. In contrast LLSVMs learn a local classifier on demand for each test point.  Ensemble methods also learn multiple classifiers and combine them to learn a global model. However, the classification model is fixed and does not change from one test point to another.

\textbf{Notation.} $[n]\defeq \{1,\ldots,n\}$. Let $\bbB(x,\sigma)$ denote a $d$ dimensional ball of radius $\sigma$ centered around $x$. Also, let $\langle a,b\rangle\defeq a^Tb$. Throughout the paper we shall use $w\in \bbR^{d+1}$ to denote a vector learned by using the training data set with a 1 appended to each training point in the data set as the $(d+1)^{\text{th}}$ dimension. If $x\in \bbR^d$ then  $\langle w,x\rangle\defeq \langle w,\bar{x} \rangle$ where $\bar{x}\in \bbR^{d+1}\defeq (x,1)$. 
Denote by $f:\cX\rightarrow \bbR$ an arbitary measurable function. Finally since most of our results are ``pointwise results'', we shall use $x_0$ to represent an arbitrary point, and all ``local quantities'' will be defined w.r.t. $x_0$.

\vspace{-5pt}
\section{Pointwise Consistency of LLSVMs \label{sec:mainthm}}
We now state the assumptions and our first main result.
\begin{compactitemize}
 \item A0: The domain $\cX\subset \bbR^d$ is compact, $||x||_2\leq M$ for all $x$ in $\cX$, and the marginal distribution on $\cX$ is absolutely continuous w.r.t. the Lebesgue measure. 
\item A1:  Let $C^1$ denote the class of functions that are at least once differentiable on $\cX$. We assume that $\eta(x)\mathop{:=} \bbP[y=1|x]\in C^1$, and as a result  $f_B(x)\defeq 2\eta(x)-1\in C^1$. Such smoothness assumptions (and stronger ones) have been used to study  minimax rates for classification in~\cite{yang1999minimax}. The impact of A1 is two fold. Firstly the minimizer of $L$ risk is a function of $\eta$. For hinge loss this function is $\sgn(2\eta(x)-1)$. The same holds true even for ``local'' versions of $L$  risk and the $0-1$ risk. Since $\eta\in C^1$, one can invoke continuity arguments, to guarantee a small enough radius $\sigma$, where the minimizer of the $L$ risk is a smooth function. Hence, one can restrict the search for an optimal function to $C^1$. The definition of such local quantities is done in Section~\eqref{sec:overview}. 
\item A2: $K(\cdot,\cdot,\cdot)$ is a finite tailed smoothing kernel function that satisfies $K(\cdot,\cdot,\cdot)\geq 0$ (positive kernel), $\int_{x_2\in \bbR^d} K(x_2,x_1,\sigma)~\mathrm{d}x_2=1$ for all $x_1\in \cX$, vanishes for all $ x\notin \bbB(x_1,\sigma)$, and $K(x_1,x_2,\sigma)\leq K_{m}=\Theta(\frac{1}{\sigma^d})$ for all $x_1,x_2\in \bbR^d$. Assumption A2 are standard assumptions from the nonparametric estimation literature~\cite{tsybakov2009introduction}. The finite tail assumption of the smoothing kernel simplifies proofs and should be easy to relax.
\item A3: For all $x_0$ in $\cX$, $\lim_{\sigma\rightarrow 0}\bbE[o(||x-x_0||)K(x,x_0,\sigma)]=0$. A3 allows us, in the limit, to approximate the minimum local $L$ risk using only linear functions, and therefore allows us to model non-linear decision boundaries via locally linear fits.  
\item A4: Let $H_\sigma$ be the region of intersection of a halfspace and $\bbB(x_0,\sigma)$, such that $\frac{\Vol(H_{\sigma})}{\Vol(\bbB(x_0,\sigma))}\geq \frac{1}{2}$. Then a.s. w.r.t. $\cD_{\cX}$, $\lim_{\sigma\rightarrow 0}\inf_{H_\sigma}\bbE_{x\sim\cD_{\cX}} K(x,x_0,\sigma)\bbm1_{H_{\sigma}}=c'_{x_0}>0$. A4 requires that the mass in $\bbB(x_0,\sigma)$ for small $\sigma$ is spread out and is not all located in a small region in $\bbB(x_0,\sigma)$.  
  \end{compactitemize}
\vspace{-9pt}
As a simple example, consider the setup where the marginal distribution has uniform density on $[-1,+1]$, and the kernel function is the Epanechnikov kernel.  Under this setting, for any $x_0\in (-1,1)$, we get $\lim_{\sigma\rightarrow 0}\bbE[o(||x-x_0||)K(x,x_0,\sigma)]=\lim_{\sigma\rightarrow 0}\frac{3}{8\sigma}\int_{x_0-\sigma}^{x_0+\sigma}|x-x_0|(1-(\frac{x-x_0}{\sigma})^2)~\mathrm{d}x\leq \lim_{\sigma\rightarrow 0}\frac{3\sigma}{16}=0$. The same result applies even for $x_0\in \{-1,+1\}$. Hence assumption A3 is satisfied. To verify the validity of A4, it is enough to see that $\lim_{\sigma\rightarrow 0}\inf_{\theta\in [0,\sigma]} \frac{3}{4\sigma}\int_{x_0-\sigma}^{x_0+\theta}(1-(\frac{x-x_0}{\sigma})^2)~\mathrm{d}x=1/4$. Hence $c'_{x_0}=1/4$. Finally, we shall work with only the hinge loss. Hence, whenever we refer to $L$ risk we basically mean the risk due to the hinge loss. We are now in a position to state our first result regarding pointwise consistency of LLSVMs.
 \begin{theorem}\label{thm:main} 
   Given an $x_0\in \cX$ , if assumptions A1-A4 hold, then there exists an $n_0\in \mathbb{N}$ such that an LLSVM that solves the  problem~\eqref{eqn:llsvm_opt} at $x_0$, agrees with the Bayes classifier at $x_0$, for all $n\geq n_0$, and appropriate $\sigma,\lambda> 0$ that satisfy $n\rightarrow \infty, \lambda,\sigma\rightarrow 0$, such that $\frac{n\lambda^2\sigma^{4d}}{\ln^{1+\theta}n}\rightarrow \infty$ for some $\theta>0$.
\end{theorem}
\subsection{Discussion of theorem 1.} Theorem~\eqref{thm:main} provides us with conditions on $n,\lambda,\sigma$ to guarantee that the learnt LLSVM makes a Bayes consistent decision at an arbitrary test point $x_0$. Like KBR, SVMs, and LPR, our results require $n$ to grow and $\lambda,\sigma$ to decay at certain rates that are precisely captured by theorem~\eqref{thm:main}. In classification, global consistency results~\cite{devroye1996ptp,steinwart2005consistency} are proved which demonstrate that that the 0-1 risk of the classifier converges to that of a Bayes classifier asymptotically. Such global consistency results are asymptotic in nature. In comparison we prove that, at an arbitrary $x_0$, we can choose sufficiently large amount of data, and appropriate parameter settings $\lambda,\sigma$ (depending on $n$) such that the LLSVMs decision matches that of the Bayes classifier at $x_0$. For these reasons it seems inappropriate to compare consistency results of SVMs with those for pointwise consistency of LLSVMs. Proving a global consistency result for LLSVMs remains an open problem, that we intend to tackle in the future. In the case of LPR, however, pointwise properties has been investigated ~\cite{tsybakov2009introduction}, such as how quickly  the squared loss of an LPR estimator at point $x_0$ converges to the squared loss of the true function. Here we are guaranteed that as $n\rightarrow \infty$, and with appropriate $\sigma$, and with any degree of the polynomial, the excess error at $x_0$ converges to 0.  However, as stated above, we can prove that we can predict the label of $x_0$ correctly with a finite amount of data. It is inappropriate to compare the results of LPR and LLSVMs, since LPR requires prediction of a real valued quantity, whereas LLSVMs are concerned with prediction of a binary label.  
As we mentioned in the related work section, the proof strategy that was used to prove the consistency of kernel based rules does not work for our case. Techniques from the literature for LPR~\cite{tsybakov2009introduction} cannot be used for proving the pointwise consistency result for LLSVMs. This is because, in LPR we are interested in the squared loss of the estimator. Squared loss allows a bias-variance decomposition, and the analysis requires the analysis of this decomposition. However, in classification we are concerned with 0-1 loss, which does not allow such a decomposition.  
\vspace{-10pt}
\subsection{Overview of Proof of Theorem~\eqref{thm:main}}\label{sec:overview} As the proof of theorem~\eqref{thm:main} is quite involved we shall first present an overview of our proof. Since the statement of theorem~\eqref{thm:main} is for each point $x_0$, we will define certain local quantities, and use them throught our proof.  Our proof has three main steps. We first establish the approximation properties of our function class. We then make a connection between 0-1 risk and the $L$ risk, since the LLSVM problem works with the $L$ risk. Finally, we need a bound on the estimation error of LLSVM , which roughly says, how good is the LLSVM objective as a proxy to the expected local $L$ risk. We shall explain these three main steps in greater detail now. We borrow some of the ideas from the proof of consistency of SVMs by Steinwart~\cite{steinwart2005consistency}, and shall make appropriate comparisons whenever required. 

The first step  (Lemma~\eqref{lem:continuity}) is to establish the local approximation properties of linear functions. In order to do so we define the \textit{regularized local} $L$ \textit{risk}, $\Rreg(w)\defeq \frac{\lambda}{2}||w||^2+\bbE[L(y\la w,x\ra)K(x,x_0,\sigma)]$, and its corresponding unregurlarized version, called \textit{local} $L$ \textit{risk}, $\Rx0(f)=\bbE[L(yf(x))K(x,x_0,\sigma)]$. The minimizer of \textit{regularized local} $L$ \textit{risk} among linear functions is denoted as $\wunderstarreg=\arg\min_{w} \Rreg(w)$.  In Lemma~\eqref{lem:continuity} we prove that for small enough $\lambda,\sigma$, the minimum of local $L$ risk among $C^1$ functions, i.e. $\inf_{f\in C^1}\Rx0(f)$ , can be well approximated by $\Rreg(\wunderstarreg)$. A similar type of result, although with global quantities, was proved by Steinwart for SVMs. However, there are two main differences. Firstly Steinwart's proof exploited the universal properties of RKHS spaces. Since we work with linear kernels which are not universal, Steinwart's arguments do not apply here. We instead use local approximation of $C^1$ functions by linear functions, which is made possible by a simple use of Taylor's expansion. Secondly while we work with $C^1$ functions, Steinwart's proof works with the space of all measurable functions. This is because their proof does not make any assumptions on the smoothness of $\eta(\cdot)$. However, our assumption A1 guarantees that it is enough to work with just $C^1$ functions. 

The second step (Lemma~\eqref{lem:L_is_good}) connects $L$ risk with 0-1 risk. In order to do so we define the local risk of a function $f$, as $\R01(f)=\bbE[\bbm1[yf(x)\leq 0]K(x,x_0,\sigma)]$. The excess local risk of $f$ is simply $\R01(f)-\inf_f\R01(f)$, which we prove in lemma~\eqref{lem:bayes} to be equal to $\R01(f)-\R01(f_B)$. In lemma~\eqref{lem:L_is_good} we prove that, for small enough $\sigma$, the difference between the local $L$ risk of a function, $f$, and a function, in $C^1$, with the smallest local risk, is an upper bound on the excess local 0-1 risk of $f$. This result is nothing but a local version of the result that was first stated in~\cite{zhang2004statistical,bartlett2006convexity}. 

In the third step, via lemmas \eqref{lem:concentration}-\eqref{lem:concentration2} we bound the deviation of the empirical local risk,~$\hatR(w)\defeq \frac{1}{n}\sum_{i=1}^n L(y_i\la w,x_i\ra)K(x_i,x_0,\sigma)$, from the local risk, $\Rx0(w)\defeq \bbE L(y\la w,x\ra)K(x,x_0,\sigma)$, for the solution of problem~\eqref{eqn:llsvm_opt}. This is done via uniform stability arguments~\cite{bousquet2002stability}. A similar result was also used by Steinwart, albeit, for global quantities. 

The fourth and final step puts together all these results to establish conditions for a.s. convergence of the sequence $\R01(\hatwunderstarreg)\defeq \bbE [\bbm1(y\la \hatwunderstarreg,x\ra)K(x,x_0,\sigma)]$ to $\inf_{f\in C^1}\Rx0(f)$. We then use this stochastic convergence along with assumption A4 to establish theorem~\eqref{thm:main}. The proof of this final step exploits the fact that $\eta$ is a continuous function. 
\begin{lemma}\label{lem:bayes}
  Let $f^{*}=\arg \underset{f}{\inf} ~\R01(f)$. Then, $\forall x\in \bbB(x_0,\sigma), f^{*}(x)\geq 0 \Leftrightarrow \eta(x) \geq \frac{1}{2}$. Hence $\R01(f^{*})=R_{x_0}(f_{B})$.
\end{lemma}
\vspace{-15pt}
\begin{proof}
  We have $\R01(f)=\bbE_{x} (\eta(x) 1(f(x)< 0)+(1-\eta(x))1(f(x)>0)) K(x,x_0,\sigma)).$
  Hence,
    $\R01(f)-\R01(f^{*})=
    \bbE_{x}(2\eta(x)-1)(1(f(x)\geq 0)-1(f^{*}(x) <0))K(x,x_0,\sigma).$
  Now by definition the above term is non-negative for all measurable functions $f$. Hence in $\bbB(x_0,\sigma)$ the behavior of $f^{*}$ is exactly the same as that of Bayes classifier. 
\end{proof}
\vspace{-10pt}
The above lemma tells us that even though the local risk uses a kernel function to weight the loss function, the minimizer of the local 0 - 1 risk, in a $\sigma$ neighborhood of $x_0$, behaves like the Bayes optimal classifier. This simple yet crucial result, would not be valid if one used a kernel that could take negative values (negative kernels). i.e. with a negative kernel it is not possible to  guarantee that $f^{*}_{x_0,\sigma}>0\Leftrightarrow \eta(x)\bbm1[x\in \bbB(x_0,\sigma)]\geq 1/2$. 
\begin{lemma} Under assumptions A1-A3, at any point $x_0\in \bbR^d$, $\wunderstarreg$ satisfies the property
\begin{equation}\lim_{\sigma\rightarrow 0}[\lim_{\lambda\rightarrow 0}\Rreg(\wunderstarreg)-\inf_{f\in C^1}\Rx0(f)]=0.
\end{equation}
\label{lem:continuity}
\end{lemma}
\vspace{-23pt}
\begin{proof}
  Step 1. We shall begin by proving the following statement.
  \begin{equation}
    \label{eqn:step1}
    \forall \sigma > 0: \lim_{\lambda\rightarrow 0} \Rreg(\wunderstarreg)=\inf_{w} \Rx0(w).
\end{equation} 
Fix a $\sigma>0$, and let $\epsilon > 0$ be given. Since $\Rx0(\cdot)$ is a continuous convex function, hence it is possible to find atleast one $w_{\epsilon,\sigma}$ with $||w_{\epsilon,\sigma}||<\infty$, such that $\Rx0(w_{\epsilon,\sigma})\leq \inf_{w} \Rx0(w)+\epsilon$


Since $\frac{\lambda}{2}||w||^2$ is continuous in $\lambda$, there exists a $\lambda(\epsilon,\sigma)$ such that for all $\lambda \leq \lambda(\epsilon,\sigma): \frac{\lambda}{2}||w_{\epsilon,\sigma}||^2\leq \epsilon$.
 Now for any $\lambda\leq \lambda_0$, we get 
\begin{multline}
  \Rreg(\wunderstarreg)\leq \Rreg(w_{\epsilon,\sigma})=\frac{\lambda}{2}||w_{\epsilon,\sigma}||^2+\bbE L(y\la w_{\epsilon,\sigma},x\ra) K(x,x_0,\sigma)\leq 2\epsilon+\inf_{w} \Rx0(w).
\end{multline}
Since $\epsilon$ was arbitrary equation~\eqref{eqn:step1} follows.\newline
Step 2. In the second step we prove that 
 $\lim_{\sigma\rightarrow 0} [\inf_{w} \Rx0(w)-\inf_{f\in C^1} \Rx0(f)]=0.$
Suppose the real valued function $g_{\sigma}$ is the minimizer of $\Rx0(f)$ for $f\in C^1$. By Taylor expansion we have $g_{\sigma}(x)=g_{\sigma}(x_0)+Dg_{\sigma}(x_0)(x-x_0)+o(||x-x_0||)$. Hence,
\begin{multline}
  \inf_w \Rx0(w)-\inf_{f\in C^1}\Rx0(f)=\inf_w \Rx0(w)-\inf_{f\in C^1}\Rx0(f)
  \leq \Rx0(w_{*})-\Rx0(g_{\sigma})\\
  =\bbE~[(L(y\la w_{*},x\ra)-L(yg_{\sigma}(x)))K(x,x_0,\sigma)]\\
  \leq \bbE~[o(||x-x_0||) K(x,x_0,\sigma)]\rightarrow 0,\nonumber
\end{multline} 
where the last step is due to A2. This completes the proof of our second part.
\end{proof}\vspace{-6pt}
\begin{lemma}\label{lem:L_is_good}
 \textbf Suppose $\eta(x_0)\neq 1/2$. Then for a sufficiently small $\sigma$, such that $\eta(x)\neq 1/2$ for any $x\in \bbB(x_0,\sigma)$, we get $\R01(f)-\inf_f\R01(f)\leq \Rx0(f)-\inf_{f}\Rx0(f)$.
\end{lemma}
\begin{proof}
  Define $\Delta=\{x|f(x)f_B(x) < 0\}, f^*_L(x)=\sgn(2\eta(x)-1)$. 
  \begin{align}
    \R01(f)-\R01& = \bbE [|2\eta(x)-1| K(x,x_0,\sigma)\bbm1_{\Delta}]\nonumber\\
      &\leqa
      \bbE[(1- \eta(x)L(f^{*}_L)-(1-\eta(x))L(f^{*}_L))K(x,x_0,\sigma)\bbm1_{\Delta}]\nonumber\\
      &\leqb \Rx0(f)-\Rx0(f^*_L))\nonumber
  \end{align}
  In step (a) we used the fact that for the hinge loss 
  $|2\eta(x)-1|\leq 1-(\eta(x) L(f_L^{*})+(1-\eta(x)) L(f_L^{*}))$, and in step (b) we used the fact that on the event $\Delta$, it is better to predict using the 0 function rather than predicting with $f$.
  \end{proof}
\vspace{-5pt}
We now need the notion of uniform stability to establish the concentration result, which were outlined in the proof overview. Roughly uniform stability~\cite{bousquet2002stability} bounds the difference in loss of a learning algorithm, at any arbitrary point, due to removal of any one point from the training dataset.
  \begin{lemma}\label{lem:concentration} 
    LLSVMs obtained by solving the optimization problem  ~\eqref{eqn:llsvm_opt} at any point $x_0$ has uniform stability of $O\left(\frac{2M^2}{n\lambda\sigma^{2d}}\right)$ w.r.t. the loss function $L(y\langle\hatwunderstarreg,x\rangle)K(x,x_0,\sigma)$.
  \end{lemma}
\vspace{-10pt}
  \begin{proof}
    Let $\hatwunderstarreg$, $\hatwunderstarminusireg$ be the LLSVMs learned at $x_0$ using data sets $S,S^{-i}$ respectively. For any $z=(x,y)\in \cX\times \{-1,+1\}$, we have
    \begin{equation}
      \label{eqn:to_prove_about_loss}
      \bigl(L(y\la \hatwunderstarreg,x\ra)-L(y\la \hatwunderstarminusireg,x\ra)\bigr)K(x,x_0,\sigma)\leq MK_{m} ||\hatwunderstarreg-\hatwunderstarminusireg||.
    \end{equation}
    Hence it is enough to bound $||\hatwunderstarreg-\hatwunderstarminusireg||$. By definition both $\hatwunderstarreg, \hatwunderstarminusireg$ are solutions of their respective convex optimization problem. 
Let
\begin{multline}\label{eqn:defn_Nw} N(w)\defeq\frac{\lambda}{2}||w-\hatwunderstarminusireg||^2+\frac{1}{n}\bigl\langle \sum_{j=1}^n dL(y_j\langle \hatwunderstarreg,x_j \rangle)K(x_j,x_0,\sigma)y_jx_j-\\\qquad\qquad\sum_{j\neq i}^n dL(y_j\langle \hatwunderstarminusireg,x_j \rangle)K(x_j,x_0,\sigma)y_jx_j,
  w-\hatwunderstarminusireg\bigr\rangle,
\end{multline}
where $dL(.)$ is an element of the subgradient of $L$ at the appropriate arguement. We have   $N(\hatwunderstarminusireg)=0,dN(\hatwunderstarreg)=0$. Hence $\hatwunderstarreg$ is an optimal solution of the minimization problem: $\min_{w} N(w)$, and we have $N(\hatwunderstarreg)\leq N(\hatwunderstarminusireg)\leq 0$. We get
\begin{equation*}\label{eqn:some_eqn3}
  \frac{\lambda}{2}||\hatwunderstarreg-\hatwunderstarminusireg||^2\leqa \frac{-1}{n}\langle dL(y_i\langle \hatwunderstarreg,x_i\rangle) K(x_i,x_0,\sigma)y_ix_i,\hatwunderstarreg-\hatwunderstarminusireg\rangle\leq \frac{MK_m}{n}||\hatwunderstarreg-\hatwunderstarminusireg|| .
\end{equation*} 
\begin{equation}
  \label{eqn:bound_on_w}
  ||\hatwunderstarreg-\hatwunderstarminusireg||\leq 2MK_m/n\lambda.
\end{equation}
where the inequality in step (a) uses properties of convex functions.
Using Equations~\eqref{eqn:to_prove_about_loss},~\eqref{eqn:bound_on_w}  we get
$\bigl(L(y\la \hatwunderstarreg,x\ra)-L(y\la \hatwunderstarminusireg,x\ra)\bigr)K(x,x_0,\sigma)\leq  O(\frac{2M^2}{n\lambda\sigma^{2d}})$.
  \end{proof} 
\begin{lemma}~\cite{bousquet2002stability}
    \label{lem:bousquet_bound}
    Let $A_S$ be the hypothesis learnt by an algorithm $A$ on dataset $S$, such that $0\leq L(A_{S},z)\leq M_1$. Suppose $A$ has uniform stability $\beta$ w.r.t $L(\cdot)$. Then, $\forall n\geq1,\delta\in (0,1)$, we have 
    \begin{equation}
      \bbP\left[R-R_{\text{emp}}\geq 2\beta+\epsilon\right]\leq \exp(-2n\epsilon^2/(4n\beta+M_1)^2). 
    \end{equation}
  \end{lemma}
  \begin{lemma}\label{lem:concentration2}
    For any point $x_0\in \cR^{d}$ we have 
    \begin{equation} \bbP\Bigl[\Rx0(\hatwunderstarreg)-\hatR(\hatwunderstarreg)\geq \frac{4M^2}{n\lambda\sigma^{2d}}+\epsilon\Bigr] \leq \exp\Biggl(\frac{-2n\lambda^2\sigma^{4d}\epsilon^2}{(8M^2+\lambda\sigma^d+M\sqrt{\lambda\sigma^d})^2}\Biggr).
    \end{equation}
  \end{lemma}
  \begin{proof}
    The desired result follows from  lemmas \eqref{lem:concentration}-\eqref{lem:bousquet_bound} and by substituting  $\hatR(\hatwunderstarreg)$ for $R_{\text{emp}}$ and $\Rx0(\hatwunderstarreg)$ for $R$ in lemma ~\eqref{lem:bousquet_bound}, and by susbtituting $M_1= O(\frac{1}{\sigma^d})+O(\frac{M}{\sigma^d\sqrt{\lambda\sigma^d}})$, which was obtained by using the fact that hinge loss is 1-Lipschitz.
  \end{proof}
  \vspace{-10pt}
  \textbf{Proof of Theorem~\eqref{thm:main}.}
  The proof is in two parts. In the first part we shall prove that under the conditions stated in the premise of the theorem $\R01(\hatwunderstarreg)\rightarrow \R01(f_B)$ a.s. The second part then uses this almost sure convergene of local risk to guarantee that $\hatwunderstarreg$ and $f_B$ agree on the label of $x_0$. Fix any $\epsilon>0$. Let $\delta_{n,\lambda,\sigma}^{(1)}\defeq\exp\left(\frac{-\epsilon^2n\sigma^{2d}}{2(1+M\sqrt{\frac{2}{\lambda}\bbE K(x,x_0,\sigma)})^2}\right), \delta_{n,\lambda,\sigma}^{(2)}\defeq\exp\left(\frac{-2n\lambda^2\sigma^{4d}\epsilon^2}{(8M^2+\lambda\sigma^{d}+M\sqrt{\lambda\sigma^d})^2}\right)$. Define $\delta_{n,\lambda,\sigma}\defeq \delta_{n,\lambda,\sigma}^{(1)}+\delta_{n,\lambda,\sigma}^{(2)}$. For appropriately chosen values of $\sigma(\epsilon),\lambda(\sigma(\epsilon))$ we have with probability atleast $1-\delta_{n,\lambda,\sigma}$
\vspace{-5pt}
  \begin{multline}
        \label{eqn:crucialbig}
	\Rreg(\hatwunderstarreg)= \frac{\lambda}{2}||\hatwunderstarreg||^2+\Rx0(\hatwunderstarreg)
        \leqa  \frac{\lambda}{2}||\hatwunderstarreg||^2+\hatR(\hatwunderstarreg)+\frac{4M^2}{n\lambda\sigma^{2d}}+\epsilon\leqb \frac{\lambda}{2}||\wunderstarreg||^2
        \\\shoveleft +\hatR(\wunderstarreg)+\frac{4M^2}{n\lambda\sigma^{2d}}+\epsilon
       \leqc \frac{\lambda}{2}||\wunderstarreg||^2+\Rx0(\wunderstarreg)+\frac{4M^2}{n\lambda\sigma^{2d}}+2\epsilon=\Rreg(\wunderstarreg)+\frac{4M^2}{n\lambda\sigma^{2d}}+2\epsilon
        \\
        \leqd \inf_{f\in C^1}\Rx0(f)+\frac{4M^2}{n\lambda\sigma^{2d}}+4\epsilon+\frac{\lambda}{2}||w_{\epsilon,\sigma}||^2+\bbE (o(||x-x_0||)K(x,x_0,\sigma)).
  \end{multline}
  In the above equations step (a) follows from lemma~\eqref{lem:concentration2}, and hence there is a failure probability of at most $\delta_{n,\lambda,\sigma}^{(1)}$. Step (b) follows from the fact that $\hatwunderstarreg$ is the minimizer of $\Rreg$, and step (c) uses the Hoeffding inequality, and incurs a failure probability of $\delta_{n,\lambda,\sigma}^{(2)}$. Choosing small enough $\sigma(\epsilon), \lambda(\sigma(\epsilon),\epsilon),$ inequality (d) follows from lemma~\eqref{lem:continuity}. Applying lemma~\eqref{lem:L_is_good} we get with probability atleast $1-\delta_{n,\lambda,\sigma}$,
\vspace{-10pt}
  \begin{multline*}
    \R01(\hatwunderstarreg)-\inf_f\R01(f)\leq \Rx0(\hatwunderstarreg)-\inf_{f\in C^1}\Rx0\leq \Rreg(\hatwunderstarreg)-\inf_{f\in C^1}\Rx0(f)
    \leqa \frac{4M^2}{n\lambda\sigma^{2d}}+\\
    4\epsilon+\frac{\lambda}{2}||w_{\epsilon,\sigma}||^2+\bbE (o(||x-x_0||)K(x,x_0,\sigma)).
  \end{multline*}
  Step (a) follows from equation~\eqref{eqn:crucialbig}, and the fact that the marginal distribution on $\cX$ is absolutely continuous. The absolute continuity gurantees that $\frac{\lambda}{2}||w_{\epsilon,\sigma}||^2\rightarrow 0$. If $n\rightarrow \infty,\lambda\rightarrow 0,\sigma\rightarrow 0, n\lambda^2\sigma^{4d}\rightarrow \infty$ we conclude that  $\R01(\hatwunderstarreg)\rightarrow \inf_{f\in C^1}\R01(f)=\R01(f_B)$ in probability. Since for data-dependent choices of $\lambda,\sigma$ that satisfy $\lambda,\sigma\rightarrow 0, \frac{n\lambda^2\sigma^{4d}}{\log^{1+\epsilon}(n)}\rightarrow \infty$, we get $\sum_{n=1}^{\infty} \delta_{n,\lambda,\sigma}<\infty$ , hence by Borel-Cantelli lemma the convegence $\R01(\hatwunderstarreg)\rightarrow \inf_f\R01(f)=\R01(f_B)$ also happens almost surely. 
\newline\newline
\vspace{-10pt}
\begin{wrapfigure}{L}{0.30\textwidth}
  \label{fig:myfig}
  \centering
  \includegraphics[scale=0.30]{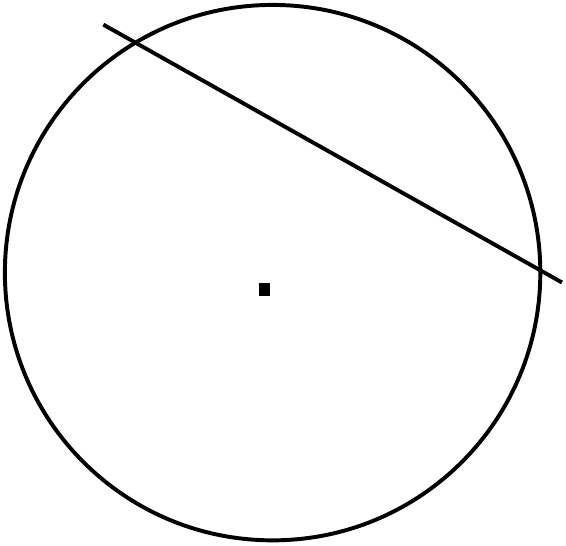}
  \caption{\footnotesize{All the points in this ball, of radius $\sigma$, centered around $x_0$ are labeled +1 by the Bayes classifier. The region of intersection between the hyperplane, and the ball, which contains the center, $x_0$, is misclassified by the hyperplane. The volume of this region is at least half of the volume of the ball.}}
    \end{wrapfigure}
  We shall now prove the second part. If $\eta(x_0)=1/2$, then the prediction of LLSVMs at point $x_0$ is irrelevant. Hence let $\eta(x_0)>1/2$. The proof is the same if $\eta(x_0)<1/2$. Choose $\sigma_1$ such that $\inf_{x\in \bbB(x_0,\sigma_1)}2\eta(x)-1\geq \frac{2\eta(x_0)-1}{2}$. Notice that because of continuity $2\eta(x)-1$ has the same sign everywhere in $\bbB(x_0,\sigma_1)$ (see Figure (1)). From A5 we are guaranteed that there exists $\sigma_2>0$ such that for all $0<\sigma\leq \sigma_2$, we have $\inf_{H_{\sigma}}\bbE K(x,x_0,\sigma)\bbm1_{H_{\sigma}}\geq \frac{c'_{x_0}}{2}$. Let $0<\sigma_0\leq \min\{\sigma_1,\sigma_2\}$. Now from the first part of the proof we know that $\R01(\hatwunderstarreg)\rightarrow \R01(f_B)$ almost surely. This guarantees that there exists a sufficiently large $n_0$ such that for appropriate $\sigma\leq \sigma_0$, and an appropriate choice of $\lambda$, we get 
    \begin{equation}\label{eqn:asguar}
      \bbP[\R01(\hatwunderstarreg)-\R01(f_B)\leq c'_{x_0}|2\eta(x_0)-1|/8]=1.
    \end{equation}
    Now for the above choice of $n_0,\lambda,\sigma$, represent by $\Delta$ the region of disagreement between $\hatwunderstarreg$ and $f_B$. Assume that $x_0\in \Delta$. Since $2\eta(x)-1$ has the same sign everywhere in $\bbB(x_0,\sigma)$, we get $\Delta=\{x\in\bbB(x_0,\sigma)|\la\hatwunderstarreg,x\ra\leq 0\}$, and hence the volume of $\Delta$ is at least half of $\bbB(x_0,\sigma)$. Hence
    \vspace{-5pt}
    \begin{multline}
      \R01(\hatwunderstarreg)-\R01(f_B)=\bbE |2\eta(x)-1|K(x,x_0,\sigma)\bbm1_{\Delta}
      \geq \frac{2\eta(x_0)-1}{2}\bbE K(x,x_0,\sigma)\bbm1_{\Delta}\\
      \geq \frac{2\eta(x_0)-1}{2}\inf_{H_{\sigma}}\bbE K(x,x_0,\sigma)\bbm1_{H_{\sigma}}
      \geq \frac{(2\eta(x_0)-1)c_{x_0}'}{4},
    \end{multline}
    which is a contradiction to equation~\eqref{eqn:asguar}. Hence $f_{B}$ and $\hatwunderstarreg$ agree on the label of $x_0$.
\vspace{-10pt}
    \section{Risk Bounds and Rates of Convergence to Stochastic Objective.}\label{sec:theory_contrib}
    LLSVMs solves a local optimization problem that can be seen as minimizing an empirical version of the stochastic objective $\Rreg(w)$. It is then natural to ask as to how quickly does the value of the stochastic objective for $w=\hatwunderstarreg$ converge to the minima of the stochastic objective? In Theorem~\eqref{thm:stochastic} we demonstrate, via stability arguments, that for an arbitrary test point $x_0$, this convergence happens at the rate of $O(1/\sqrt{n}\lambda\sigma^{2d})$. In Theorem~\eqref{thm:risk_bound_for_llsvm} we establish generalization bounds for a global classifier learnt by solving LLSVM's at any randomly chosen point $x$, in terms of the empirical error of LLSVMs. Due to lack of space the proofs are postponed to the supplement.
    \begin{theorem}\label{thm:stochastic}
      With probability at least $1-\delta$ over the random input training set we have 
      \begin{equation} \Rreg(\hatwunderstarreg)-\Rreg(\wunderstarreg)\leq \tilde{O}\left(\frac{1}{\sqrt{n}\lambda\sigma^{2d}}\right). 
      \end{equation}
    \end{theorem}
\vspace{-10pt}
\textbf{Discussion of theorem~\eqref{thm:stochastic}}. In theorem~\eqref{thm:stochastic}, it might be possible to improve the dependence on $n$ from $\frac{1}{\sqrt{n}}$ to $\frac{1}{n}$ via the peeling idea~\cite{sridharan2008fast}. Based on ~\cite{sridharan2008fast}, we conjecture that the dependence on $\lambda$ is optimal, while the dependence on $\sigma$ may be improved from $1/\sigma^{2d}$ to $1/\sigma^d$.
    \begin{theorem}\label{thm:risk_bound_for_llsvm}
      Let $\hatwunderstarreg(x)$ be the vector obtained by solving the LLSVM problem, with parameters $\lambda,\sigma$, at a randomly drawn point $x$. With probability at least $1-\delta$ over the random sample, we have
      \begin{equation*}
	\bbE L(y\langle\hatwunderstarreg(x),x\rangle)\leq \frac{1}{n}\sum_{i=1}^n L(y_i\langle \hatwunderstarreg(x_i),x_i\rangle)+\frac{4M^2}{n\lambda\sigma^d}+(1+O(M\sqrt{1/\lambda\sigma^d}))\sqrt{\mathrm{ln}(1/\delta)/2n}.
      \end{equation*}
    \end{theorem}
\vspace{-10pt}
    \textbf{Discussion of Theorem~\eqref{thm:risk_bound_for_llsvm}}. Without any further noise assumptions, the dependence on $n,\lambda,\sigma$ is optimal. With the Tsybakov's~\cite{boucheron2005theory} noise assumption, it is possible to improve the dependence on $n$.  The exponential dependence on $d$ is expected, and is typical of nonparametric methods.
    \vspace{-10pt} 
\section{Proofs of Theorems 8,9}
For convenience we shall begin with a risk bound from~\cite{bousquet2002stability}. This risk bound relies on the notion of uniform stability. For any learning algorithm A that learns a function $A_{S}$ after having trained on the dataset $S$ the uniform stability quantifies the absolute maginitude of the change in loss suffered by the algorithm at any arbitrary point in the space if an arbitrary $x_i$ is removed from the training dataset. The precise definition is as follows 
\begin{definition}\label{def:defn_uniform_stability}~\cite{bousquet2002stability}
  An algorithm A has uniform stability $\beta$ w.r.t the loss function $L$ if:
  \begin{equation}
    \forall S , \forall i\in \{1,\ldots,n\}, ||L(A_S,\cdot)-L(A_{S_{-i}},\cdot)||_{\infty}\leq \beta
  \end{equation}
\end{definition}
  \begin{lemma}\label{lem:bousquet_bound}
    Let $A$ be an algorithm with uniform stability $\beta$ w.r.t a loss function $0\leq L(A_{S},(x,y))\leq M_1$, for all $z\defeq(x,y)$ and all set S. Then for any $n\geq 1$, and any $\delta\in (0,1)$, the following bound holds true with probability atleast $1-\delta$ over the random draw of the sample S.
    \begin{equation}
      \bbE_{z\sim \cD}L(A_{S},z)\leq \frac{1}{n}\sum_{i=1}^n L(A_s,z_i)+2\beta+(4n\beta+M_1)\sqrt{\frac{log(1/\delta)}{2n}}.
    \end{equation}
  \end{lemma}
\setcounter{theorem}{7}
  \begin{theorem}
With probability $1-\delta$ over the random input training set we have 
\begin{equation}
  \Rreg(\hatwunderstarreg)-\Rreg(\wunderstarreg)\leq 2\beta+(4n\beta+M_1)\sqrt{\frac{log(\frac{1}{\delta})}{2n}},
\end{equation}
where 
\begin{align}
\beta&\leq \frac{2MK_{m}}{n\lambda}\Biggl[\sqrt{\frac{2\lambda L(0)}{n}}\sqrt{\sum_{j=1}^n K(x_j,x_0,\sigma)}+MK_{m}\Biggr]\\
M_1&\leq\frac{L(0)}{n}\sum_{j=1}^n K(x_j,x_0,\sigma)+L(0)K_m+K_mM\sqrt{\frac{2L(0)}{n\lambda}\sum_{j=1}^n K(x_j,x_0,\sigma)}
\end{align} 
  \end{theorem}
  \begin{proof}
The proof is via stability arguments. Let $z\defeq (x,y)$. Consider the loss function 
\begin{multline}\label{eqn:defn_q}
  q(w,z)=\frac{\lambda}{2}||w||^2+L(y\langle w,x\rangle)K(x,x_0,\sigma)-\Bigl [\frac{\lambda}{2}||\wunderstarreg||^2+L(y\langle \wunderstarreg,x\rangle)K(x,x_0,\sigma)\Bigr].
\end{multline}
It is enough to bound the stability of LLSVM's w.r.t the above loss function and also upper bound the above loss. In order to upper bound the stability of LLSVM's it is enough to upper bound for all $S, (x,y)$ the quantity $|q(\hatwunderstarreg,(x,y))-q(\hatwunderstarminusireg,(x,y))|$, 
 where $S^{-i}$ is the dataset obtained from $S$ by deleting the point $(x_i,y_i)$, and $\langle\hatwunderstarminusireg,x\rangle$ is the LLSVM learnt at $x_0$ with $S^{-i}$.
From Equation~(\ref{eqn:defn_q}) it is clear that $q(w,z)$ is $\lambda$ strongly convex in $w$ in $L_2$ norm. Hence by strong convexity 
\begin{multline}\label{eqn:strong_convexity1}
  q(\hatwunderstarreg,z)\geq q(\hatwunderstarminusireg,z)+(\hatwunderstarreg-\hatwunderstarminusireg)^T\partial q(\hatwunderstarminusireg,z)+\\\frac{\lambda}{2}||\hatwunderstarreg-\hatwunderstarminusireg||^2
\end{multline}
Similarily we have 
\begin{multline}\label{eqn:strong_convexity2}
  q(\hatwunderstarminusireg,z)\geq q(\hatwunderstarreg,z)+(\hatwunderstarminusireg-\hatwunderstarreg)^T\partial q(\hatwunderstarreg,z)+\\\frac{\lambda}{2}||\hatwunderstarreg-\hatwunderstarminusireg||^2
\end{multline}
From equations ~(\ref{eqn:strong_convexity1},\ref{eqn:strong_convexity2}) we get
\begin{multline}\label{eqn:lower_bound_upper_bound}
 (\hatwunderstarminusireg-\hatwunderstarreg)^T\partial q(\hatwunderstarreg,z)+\frac{\lambda}{2}||\hatwunderstarreg-\hatwunderstarminusireg||^2 \leq\\ q(\hatwunderstarminusireg,z)-q(\hatwunderstarreg,z)\leq \\(\hatwunderstarminusireg-\hatwunderstarreg)^T\partial q(\hatwunderstarminusireg,z)-\frac{\lambda}{2}||\hatwunderstarreg-\hatwunderstarminusireg||^2
\end{multline}
We shall now upper and lower bound the rightmost and the leftmost terms respectively. Doing this will enable us to bound the stability. Differentiating Equation~(\ref{eqn:defn_q}) w.r.t $w$ we get
\begin{align}\label{eqn:derivative_q}
  \partial q(\hatwunderstarminusireg,z)=\lambda\hatwunderstarminusireg+\partial L (y\langle \hatwunderstarminusireg,x\rangle)yK(x,x_0,\sigma)x\\
   \partial q(\hatwunderstarreg,z)=\lambda\hatwunderstarreg+\partial L (y\langle \hatwunderstarreg,x\rangle)yK(x,x_0,\sigma)x
\end{align}
Now in order to bound the rightmost term of equation~(\ref{eqn:lower_bound_upper_bound}) we use equation~(\ref{eqn:derivative_q}) to get
\begin{multline}\label{eqn:chain_of_inequalities}
  (\hatwunderstarminusireg-\hatwunderstarreg)^T\partial q(\hatwunderstarminusireg,x)-\frac{\lambda}{2}||\hatwunderstarminusireg-\hatwunderstarreg||^2\leq\\ (\hatwunderstarminusireg-\hatwunderstarreg)^T\Bigl[\lambda\hatwunderstarminusireg+
  \partial L(y\langle \hatwunderstarminusireg\rangle)yK(x,x_0,\sigma)x\Bigr]\leq \\ \qquad\qquad\qquad||\hatwunderstarminusireg-\hatwunderstarreg||~ ||\lambda\hatwunderstarminusireg+\partial L(y\langle \hatwunderstarminusireg\rangle)yK(x,x_0,\sigma)x||
\end{multline}
where the last inequality follows from Cauchy-Schwartz inequality. We shall begin by bounding $||\hatwunderstarminusireg-\hatwunderstarreg||$. Now by the definition of $\hatwunderstarminusireg,\hatwunderstarreg$ we get
\begin{align}
\lambda\hatwunderstarreg+\frac{1}{n}\sum_{j=1}^n L(y_j\langle\hatwunderstarreg,x_j \rangle)K(x_j,x_0,\sigma)y_jx_j=0\label{eqn:derivative_equations1}\\
\lambda\hatwunderstarminusireg+\frac{1}{n}\sum_{\substack{j=1\\j\neq i}}^n L(y_j\langle\hatwunderstarreg,x_j \rangle)K(x_j,x_0,\sigma)y_jx_j=0\label{eqn:derivative_equations2}
\end{align}
Now consider the following convex optimization problem
\begin{multline}
N(w)=\frac{\lambda}{2}||w-\hatwunderstarminusireg||^2+\frac{1}{n}\langle  \sum_{j=1}^n \partial L(y_j\langle w,x_j \rangle)K(x_j,x_0,\sigma)y_jx_j- \\\sum_{\substack{j=1\\j\neq i}}^n \partial L(y_j\langle\hatwunderstarreg,x_j \rangle)K(x_j,x_0,\sigma)y_jx_j,w-\hatwunderstarminusireg\rangle
\end{multline}
  \end{proof}
  It is trivial to verify using equations~(\ref{eqn:derivative_equations1},\ref{eqn:derivative_equations2}) that $\frac{\partial N(\hatwunderstarreg)}{\partial w}=0$, and hence from convex analysis we know that $\hatwunderstarreg$ is the optimal solution of the convex optimization problem $N(w)$. Also $N(\hatwunderstarreg)\leq N(\hatwunderstarminusireg)= 0$. Hence we get 
  \begin{multline}
  \frac{\lambda}{2}||\hatwunderstarreg-\hatwunderstarminusireg||^2\leq\\
  \frac{-1}{n}\sum_{\substack{j=1\\j\neq i}}\langle\partial
  L(y_j\langle\hatwunderstarreg,x_j\rangle)y_jK(x_j,x_0,\sigma)x_j\\
  \qquad \qquad \qquad\qquad\qquad-\partial L(y_j\langle\hatwunderstarminusireg,x_j\rangle)y_jK(x_j,x_0,\sigma)x_j,
  \hatwunderstarreg-\hatwunderstarminusireg\rangle\\\qquad \qquad\qquad\quad-\frac{1}{n}\langle \partial L(y_i\langle \hatwunderstarreg,x_i\rangle)y_iK(x_i,x_0,\sigma)x_i,\hatwunderstarreg-\hatwunderstarminusireg\rangle\leq\\ \qquad\qquad\qquad\qquad -\frac{1}{n}\langle \partial L(y_i\langle \hatwunderstarreg,x_i\rangle)y_iK(x_i,x_0,\sigma)x_i,\hatwunderstarreg-\hatwunderstarminusireg\rangle \leq \\\frac{1}{n}MK(x_i,x_0,\sigma)||\hatwunderstarreg-\hatwunderstarminusireg||
  \end{multline}
  where the second inequality is due to the fact that $L(\cdot)$ is a
  convex loss function, and hence $(dL(b)-dL(a))(b-a)\geq 0$, and the last inequality due to Cauchy-Schwartz and the fact that $L(\cdot)$ is $1$ Lipschitz. Hence we get 
  \begin{equation}\label{eqn:diff_norm_bound}
  ||\hatwunderstarminusireg-\hatwunderstarreg||\leq \frac{2}{n\lambda}MK(x_i,x_0,\sigma).
  \end{equation}
  Finally we have by the optimality of $\hatwunderstarminusireg,\hatwunderstarreg$
  \begin{align}
  \label{eqn:norm_bound}
  \frac{\lambda}{2}||\hatwunderstarminusireg||^2\leq \frac{1}{n}\sum_{\substack{j=1\\j\neq i}}^n L(0)K(x_i,x_0,\sigma)\\
  \frac{\lambda}{2}||\hatwunderstarreg||^2\leq \frac{1}{n}\sum_{j=1}^n L(0)K(x_j,x_0,\sigma)
  \end{align}
Using equations~(\ref{eqn:lower_bound_upper_bound},\ref{eqn:derivative_q},\ref{eqn:chain_of_inequalities},\ref{eqn:diff_norm_bound},\ref{eqn:norm_bound}) we get 
\begin{multline}\label{eqn:lower_bound_q}
  q(\hatwunderstarminusireg,z)-q(\hatwunderstarreg,z)\leq \\\frac{2MK(x_i,x_0,\sigma)}{n\lambda}\Bigl[\sqrt{\frac{2\lambda L(0)}{n}}\sqrt{\sum_{\substack{j=1\\j\neq i}}^n K(x_j,x_0,\sigma)}+MK(x,x_0,\sigma)\Bigr]
\end{multline}
  One can use similar techniques to lower bound the leftmost term in Equation~(\ref{eqn:lower_bound_upper_bound}) to get 
    \begin{multline}\label{eqn:upper_bound_q}
      q(\hatwunderstarminusireg,z)-q(\hatwunderstarreg,z)\geq\\ 
      -\frac{2M}{n\lambda}K(x_i,x_0,\sigma)\Bigl[\sqrt{\frac{2\lambda L(0)}{n}}\sqrt{\sum_{j=1}^n K(x_j,x_0,\sigma)}+MK(x,x_0,\sigma)\Bigr]
    \end{multline}
  Using the fact that $\beta=\sup_{S,z} |q(\hatwunderstarreg,z)-q(\hatwunderstarreg,z)|$ and equations~(\ref{eqn:lower_bound_q},\ref{eqn:upper_bound_q}) we get 
  \begin{equation}
  \beta\leq \frac{2MK_{m}}{n\lambda}\Bigl[ \sqrt{\frac{2\lambda L(0)}{n}}\sqrt{\sum_{j=1}^n K(x_j,x_0,\sigma)}+MK_{m}\Bigr]
  \end{equation}
  In order to apply theorem~(\ref{lem:bousquet_bound}) it is enough to  upper bound $q(\hatwunderstarreg,z)$. We have
  \begin{multline}
q(\hatwunderstarreg,z)\leq  \frac{\lambda}{2}||\hatwunderstarreg||^2+L(y\langle \hatwunderstarreg,x\rangle)K(x,x_0,\sigma)\leq \\\frac{L(0)}{n}\sum_{j=1}^n K(x_j,x_0,\sigma)+ L(y\langle \hatwunderstarreg,x\rangle)K(x,x_0,\sigma)\leq\\ \frac{L(0)}{n}\sum_{j=1}^n K(x_j,x_0,\sigma)+L(0)K_m+K_mM\sqrt{\frac{2L(0)}{n\lambda}\sum_{j=1}^n K(x_j,x_0,\sigma)}.
    \end{multline}
  Now applying theorem~(\ref{lem:bousquet_bound}) to LLSVM's with the
  loss function $q(A_s,z)$ and since $\hatRreg(\hatwunderstarreg)\leq
  \hatRreg(\wunderstarreg)$ we get the desired result.

\begin{theorem}\label{thm:risk_bound_for_llsvm}
   Let $\hatwunderstarreg(x)$ be the solution obtained by solving the
   LLSVM problem at $x$. With probability at least $1-\delta$ over the random sample for an LLSVM, we have
      \begin{equation*}
	\bbE L(y\langle\hatwunderstarreg(x),x\rangle)\leq \frac{1}{n}\sum_{i=1}^n L(y_i\langle \hatwunderstarreg(x_i),x_i\rangle)+\frac{4M^2}{n\lambda\sigma^d}+\Bigl(1+O(M\sqrt{1/\lambda\sigma^d})\Bigr)\sqrt{\frac{\mathrm{ln}(1/\delta)}{2n}}.
      \end{equation*}
    \end{theorem}
    \begin{proof}
      By lemma~\eqref{lem:bousquet_bound} we are done if we can upper bound the loss suffered by LLSVMs at any point, and the stability of LLSVMs w.r.t the loss $L(y\langle\hatwunderstarreg(x),x\rangle)$. We have $|L(y\langle\hatwunderstarreg(x),x\rangle)-L(y\langle\hatwunderstarminusireg(x),x\rangle)|= O\left(2M^2/n\lambda\sigma^d\right)$,
      where we used the upper bound on
      $||\hatwunderstarreg(x)-\hatwunderstarminusireg(x)||$ presented in
      Equation 9 of Lemma 5 in the main paper. Finally  $L(y\langle\hatwunderstarreg(x),x\rangle)\leq 1+M||\hatwunderstarreg(x)||\leq 1+O(M\sqrt{1/\lambda\sigma^d})$. Apply lemma~\eqref{lem:bousquet_bound} with $\beta=O(2M^2/n\lambda\sigma^d)$ and $M_1=1+O(M\sqrt{1/\lambda\sigma^d})$ to finish the proof.
    \end{proof}
\section{Discussion and Open Problems}
    \vspace{-5pt}
    Our results guarantee that the decision of an LLSVM learnt at $x_0$ matches that of the Bayes classifier after having seen enough data. An important open problem is to establish global Bayes consistency of LLSVMs. It is not clear to us if the pointwise consistency result can be used to do so. Theorem~\eqref{thm:risk_bound_for_llsvm} currently does not exploit our large margin formulation. A natural extension of this theorem would be to establish a result that depends on some kind of a local notion of margin. Our current results depend on the dimensionality of the ambient space. It should be possible, under appropriate manifold assumptions~\cite{ozakin2009submanifold},~\cite{yu2009nonlinear} to improve this dependency to use the intrinsic dimension.
\bibliographystyle{unsrt}
\footnotesize{\bibliography{../../all}}
\end{document}